\documentclass[wcp]{jmlr}
\usepackage{amsmath,amssymb,graphicx,url}
\jmlrvolume{266}
\jmlryear{2025}
\jmlrworkshop{Conformal and Probabilistic Prediction with Applications}
\jmlrproceedings{PMLR}{Proceedings of Machine Learning Research}

\usepackage{hyperref}
\usepackage{url}

\usepackage{amsmath}
\usepackage{amscd}
\usepackage{wrapfig}
\usepackage{color, colortbl}
\usepackage{arydshln}
\usepackage{nicematrix,tikz}
\usepackage{algorithm}

\usepackage{mathtools}
\usepackage{nicefrac}
\usepackage{algorithmic}

\usepackage[capitalize,noabbrev]{cleveref}

\newcommand{\lt}{<}
\newcommand{\gt}{>}
\usepackage{booktabs}

\title{Privacy-Preserving Conformal Prediction Under Local Differential Privacy}
\author{
\Name{Coby Penso}\Email{coby.penso24@gmail.com}\\
\addr{Bar-Ilan University, Israel}\\
\Name{Bar Mahpud}\Email{mahpudb@biu.ac.il}\\
\addr{Bar-Ilan University, Israel}\\
\Name{Jacob Goldberger}\Email{jacob.goldberger@biu.ac.il}\\
\addr{Bar-Ilan University, Israel}\\
\Name{Or Sheffet}\Email{or.sheffet@biu.ac.il}\\
\addr{Bar-Ilan University, Israel}}

\editor{Khuong An Nguyen, Zhiyuan Luo, Tuwe Löfström, Lars Carlsson and Henrik Boström}

\begin{document}
\maketitle

\begin{abstract}

Conformal prediction (CP) provides sets of candidate classes with a guaranteed probability of including the true class. However, it typically relies on a calibration set with clean labels. We address privacy-sensitive scenarios where the aggregator is untrusted and can only access a perturbed version of the true labels.
We propose two complementary approaches under local differential privacy (LDP).
In the first approach, users do not access the model but instead provide their input features and a perturbed label using a \(k\)-ary randomized response. In the second approach, which enforces stricter privacy constraints, users add noise to their conformity score by binary search response. This method requires access to the classification model but preserves both data and label privacy. Both approaches compute the conformal threshold directly from noisy data without accessing the true labels.
We prove finite-sample coverage guarantees and demonstrate robust coverage even under severe randomization. This approach unifies strong local privacy with predictive uncertainty control, making it well-suited for sensitive applications such as medical imaging or large language model queries, regardless of whether users can  (or are willing to)  compute their own scores.

\end{abstract}

\section{Introduction}

In safety-critical applications of machine learning, it is essential for models to only make predictions when they are confident. One effective approach is to provide a (preferably small) set of possible class candidates that includes the true class with a predetermined level of certainty. This method is particularly well-suited to domains such as medical diagnosis systems, where physicians make the final decisions. By narrowing the possible diagnoses down to a manageable set, this approach supports practitioners while maintaining a controlled probability of error. A general technique for producing such prediction sets without making strong assumptions about the data distribution is called Conformal Prediction (CP) \citep{vovk2005conformal,angelopoulos2023conformal}. CP generates a prediction set with a formal guarantee that the true class is included with at least the specified confidence level. The objective is to minimize the size of the prediction set while upholding this confidence guarantee.
 The appeal of CP lies in its ability to provide distribution-free guarantees, making it widely applicable across diverse domains, including safety-critical areas such as medical diagnostics \citep{lu2022improving,olsson2022estimating} and fairness-driven applications \citep{lu2022fair}.

CP typically relies on having access to a labeled calibration set to set the CP threshold. However, in many real-world scenarios, such clean labels may be unavailable due to data privacy concerns.  For instance, in scenarios involving medical or personal data, or if the data aggregator (e.g.\ a cloud-based ML service) is considered \emph{untrusted}, privacy constraints may prohibit access to true labels and data, requiring a privacy-preserving mechanism that provides \emph{Local Differential Privacy} (LDP). In such situations, the aggregator might only be allowed to view \emph{noisy versions} of the labels. LDP has gained popularity as a strong privacy paradigm that enables data owners to randomize their data locally before sharing it with an untrusted aggregator, thus ensuring that sensitive information remains protected (e.g.\ Google's RAPPOR \citep{erlingsson2014rappor} and Apple's locally private data collection of emojis and usage patterns \citep{apple2017privacy}).

This paper tackles the challenge of applying conformal prediction when calibration-set labels (or conformity scores) must be protected through privacy-preserving mechanisms. Specifically, we introduce a Local Differentially Private Conformal Prediction (LDP-CP) framework that balances privacy with real-world considerations such as user-side computational capacity, aggregator trustworthiness, and intellectual property.

 We present two complementary LDP-CP solutions. LDP-CP-L locally perturbs labels using a randomized response, which shifts all score-related computations to the aggregator. This design suits cases where users have minimal computational resources, or the model’s internal structure is not disclosed to them, while not sharing their sensitive labels. However,  it only achieves label-DP \citep{beimel2013private, ghazi2021}. In contrast, LDP-CP-S allows users to generate and locally randomize their conformity scores, which is ideal for scenarios where both feature and label privacy are paramount and the user can handle additional computational tasks. We also offer guidelines on choosing which method aligns better with specific privacy goals, resource constraints, and per-dataset properties such as sample size and the number of classes.
By considering both score-based and label-based perturbations, we provide a flexible framework that adapts to various privacy budgets and computational setups. This framework guarantees valid coverage for the true labels asymptotically and in finite samples. As a result, LDP-CP aligns well with modern demands for secure, decentralized data handling.

\paragraph{Relationship to Past Work.} Conformal prediction in the presence of label noise has received growing attention \citep{einbinder2022conformal,sesia2023adaptive,  Clarkson2024,Penso_2024},
 with methods that adjust the calibration threshold when a known noise matrix corrupts labels. Simultaneously, LDP has become a leading approach for protecting sensitive data at the user end \citep{warner1965randomized,kairouz2016discrete,wang2017locally}, allowing users to randomize their own inputs (e.g., labels) before sharing with an untrusted aggregator. We bridge these lines of work by recognizing that LDP can be \emph{seen as a known noisy channel} on the labels or conformal scores, we can plug that channel into a ``noise-aware" conformal procedure. The end result is a conformal predictor whose coverage remains valid, while also protecting each user's label or score via $\varepsilon$-LDP. To the best of our knowledge, there is no prior work on LDP conformal prediction. Thus, we are the first to study local differential private conformal prediction. The closest work to ours is that of \citet{angelopoulos2022private}, which suggests a centrally differentially private conformal prediction procedure where a trusted aggregator has access to raw data. In our setting, the aggregator never observes true labels directly, hence, closing a key gap in privacy-preserving conformal prediction research. The contribution of this paper is threefold:
    we introduce two complementary LDP-CP methods that accommodate different privacy constraints and computational setups,
    we prove coverage guarantees for both methods under LDP constraints and
    we demonstrate the effectiveness of our approaches in privacy-sensitive applications, such as medical data analysis and untrusted cloud-based ML services.
This work bridges the gap between privacy-preserving mechanisms and conformal prediction, by providing a foundation for robust and private uncertainty quantification in real-world scenarios. Unlike other methods that rely on trusted aggregators, our approaches ensure privacy directly at the user level, which aligns with modern demands for decentralized privacy protection.

\section{Background}
\subsection{Conformal Prediction (CP)}
We consider a classification setting with $k$ classes. For an input $x$, the model outputs a probability vector $\{p(y=i \mid x;\theta)\}_{i=1}^k$. A \emph{conformity score} $S(x,y)$ measures how well the model's predicted distribution aligns with the true label $y$ (lower scores indicate  better alignment).
Some of the most popular score examples are {HPS} \citep{vovk2005conformal}:
$$
           S_{\text{HPS}}(x,i) \;=\; 1- p_i, \text{  where } p_i = p(y=i|x;\theta),$$ 
       and {APS} \citep{romano2020classification}:
      $$ S_{\text{APS}}(x,i)=\sum_{\{j:\;p_j \ge p_i\}}\,p_j + u \cdot p_i,
    \,  u \sim U[0,1].
   $$
The CP prediction set of a data point $x$ is defined as $C_q(x)=\{y| S(x,y) \le q\}$ where $q$ is a threshold that is found using a labeled calibration set $(x_1,y_1),...,(x_n,y_n)$. The CP theorem states that if we set  $q$ to be the $(1\!-\!\alpha)$ quantile of the conformal scores $S(x_1,y_1),...,S(x_n,y_n)$  we can guarantee that $1\!-\!\alpha \le p( y\in C_q(x)) \le 1\!-\!\alpha + \frac{1}{n+1}$,
where $x$ is a test point and $y$ is its unknown true label \citep{vovk2005conformal}.

\subsection{Conformal Prediction with Noisy Labels}
Assume we are given a calibration set with noisy labels. 
 Then, applying the standard CP procedure on this   set, no longer yields correct coverage for the true labels. 
Recent studies \citep{einbinder2022conformal,sesia2023adaptive, Clarkson2024,Penso_20242,Penso_2024} address this challenge by adjusting the calibration procedure to account for the noise. 
Assume a uniform  noise model with parameter $\beta\in[0,1]$:
\begin{equation}
\label{eq:flip}
p(\tilde{y} \mid y)
=\,
{1}_{\{\tilde{y}=y\}}(1-\beta) + \frac{\beta}{k},
\end{equation}
where $y$ is the true label and $\tilde{y}$ is its noisy version.
\citet{Penso_2024}  used the noisy calibration set to compute an approximation $\hat{F}^c(q)$ of the clean label distribution $p(S(x,y)<q)$. They showed that we can get the following finite sample coverage guarantee  that depends on the noise level $\beta\in[0,1]$: 
\begin{theorem} 
    Assume you have a noisy calibration set of size $n$ with noise level $\beta$,  set $\Delta(n,\beta,\delta)=\sqrt{\frac{\log(4/\delta)}{2nh^2}}$ s.t. $h=\frac{1-\beta}{1+\beta}$ and that you pick $q$ such that $\hat{F}^c(q) \ge 1-\alpha$. Then with probability at least $1-\delta$ (on the random calibration set), we have that if $(x,y)$ are sampled from the clear label distribution we get:
    $$ p(y\in C_q(x)) \ge  1-\alpha - \Delta.$$
\label{theorem:nacp}
\end{theorem}
Crucially, $\Delta(n,\beta,\delta)$ shrinks as the calibration size $n$ grows, making the coverage error negligible for sufficiently large datasets.
In the case that a coverage of at least $1-\alpha$ must be satisfied, we can apply a noise-robust CP with $1-\alpha+\Delta$ instead of $1-\alpha$.

\subsection{Local Differential Privacy (LDP)}
Traditional (central) differential privacy \citep{dwork2006differential} presumes a trusted curator who sees the raw data and then adds noise before publishing. In contrast, LDP \citep{duchi2013local,kairouz2016discrete, kasiviswanathan2011can} treats the aggregator as \emph{untrusted}: individual users randomize their own data \emph{locally} before sending it to the aggregator, thus ensuring strong privacy. LDP is considered a harder setting since noise insertion is done on the user side in a distributed manner, whereas in the centralized DP model the curator holds the entire data and can apply operations on the clean data.

\begin{definition}
    A discrete randomized mechanism $Q(\cdot)$ is $\varepsilon$-LDP if for any pair of input labels $y,y'\in\mathcal{Y}$ and any output $z$,
    \[
    Q(z\mid y) \;\le\; e^{\varepsilon}\,Q(z\mid y').
    \]
\end{definition}
This means that any two possible labels are (roughly) indistinguishable from the aggregator's perspective. A common mechanism is the \emph{$k$-ary randomized response} ($k$-RR) \citep{warner1965randomized,wang2017locally}. For a label $y\in\{1,\dots,k\}$, it outputs a noisy version $\tilde{y}$:
\[
p(\tilde{y}|y) =
\begin{cases}
 \tfrac{e^\varepsilon}{(k-1) + e^\varepsilon} & \mbox{if} \hspace{0.5cm}  \tilde{y}=y\\ \\
\tfrac{1}{(k-1) + e^\varepsilon} &
\mbox{if} \hspace{0.5cm}
\tilde{y} \ne y.
\end{cases}
\]
Translating $k$-RR into the format of a uniform noise model, we get Eq. (\ref{eq:flip}) with \(\beta = \frac{k}{k-1+e^{\epsilon}} \)
such that with probability $\beta$ label drawn uniformly and with $1-\beta$ remain unchanged.
This preserves label privacy, such that the aggregator cannot easily guess the user's true label from $\tilde{y}_i$. The parameter $\varepsilon$, known as the \emph{privacy loss}, controls the strength of the mechanism (lower $\varepsilon$ implies more noise and therefore a stronger privacy guarantee). Note that if $k=2$, we recover \emph{Warner's binary randomized response} (RR) \citep{warner1965randomized}, flipping the label with some probability.

\section{Local-DP Conformal Prediction}
\label{sec:ldp-approaches}

In this section we describe two complementary CP approaches under local differential privacy.
In the first approach, users do not access the model but instead provide their input features and a perturbed label. In the second approach, which enforces stricter privacy constraints, users add noise to their conformity score by binary search response. This method requires access to the classification model but preserves both data and label privacy.

\subsection{Local-DP Conformal Prediction on Labels (LDP-CP-L)}
\label{subsec:ldp-cp-labels}

\begin{figure}
    \centering
    \includegraphics[trim=0cm 1cm 0cm 0cm, width=0.80\linewidth]{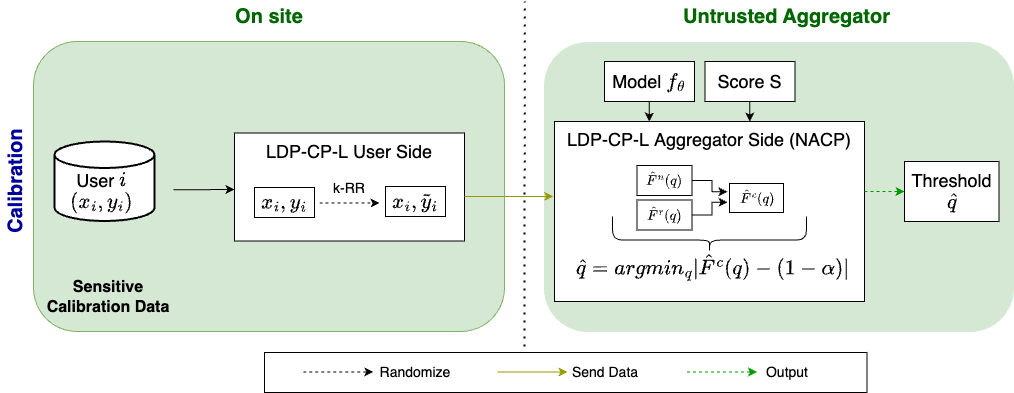}
    \caption{Local Differential Private Conformal Prediction on Labels (LDP-CP-L).}
    \label{fig:main_pipeline_l}
\end{figure}

In many real-world contexts, users lack the ability—or permission—to compute model-based scores on their own devices. This may be due to limited computational resources, a restricted API that only provides predictions, or a proprietary model architecture. To address these cases, we propose an LDP mechanism that randomizes each user’s label and then applies a noise-aware conformal calibration at the aggregator. This design protects sensitive labels while remaining model-agnostic, since the aggregator performs all the scoring steps. Figure \ref{fig:main_pipeline_l} illustrates the overall pipeline.
We next unpack LDP-CP-L by first describing the user's side procedure that applies to its data point, and then the procedure done by the untrusted aggregator to compute the conformal prediction threshold.

\begin{algorithm}
\caption{LDP-CP-L - User-side Procedure}
\label{alg:ldp_cp_labels_user}
\begin{algorithmic}[1]
\STATE \textbf{Input:} Calibration set $\{(x_i, y_i)\}_{i=1}^{n}$, privacy parameter $\epsilon$, number of classes $k$
\STATE \textbf{Output:}A noisy calibration set $\{(x_i, \tilde{y}_i)\}_{i=1}^{n}$
\FOR {each user $i \in \{1, \dots, n\}$}
    \STATE Apply $k$-ary Randomized Response (RR) to  label $y_i$ to obtain $\tilde{y}_i$: \\
    \(
    p(\tilde{y}_i \mid y_i)
    =\,
    {1}_{\{\tilde{y}_i=y_i\}}(1-\beta) + \frac{\beta}{k}
      \hspace{0.4cm}  \mbox{s.t.} \hspace{0.4cm}  \beta = \frac{k}{k-1+e^{\epsilon}}. \)
    \STATE Send $(x_i, \tilde{y}_i)$ to the aggregator.
\ENDFOR
\end{algorithmic}
\label{alg:1}
\end{algorithm}

\textbf{User's side Procedure.} Each user \(i\) with pair $(x_i, y_i)$
 first applies the  LDP mechanism - $k$-ary RR to the label $y_i$ which yields $\tilde{y}_i$
 using a uniform noise level  $\beta = \frac{k}{k-1+e^{\epsilon}}$.
 Then the user reports $(x_i, \tilde{y}_i)$ to the aggregator.

\textbf{Aggregator's Procedure.} Once the aggregator collects $\{x_i, \tilde{y}_i\}_{i=1}^{n}$ (it does not see $y_i$), it runs a ``noise-aware'' conformal method on the pairs $(x_i,\tilde{y}_i)$ \citep{Penso_2024}. Since the aggregator knows the noise model, it can estimate the coverage for the \emph{true} label.

A key advantage of this label-perturbation strategy is that the aggregator never observes raw labels, thus meeting label-LDP guarantees. Meanwhile, by modeling a  $k$-ary Randomized Response as a known noise channel, the aggregator recovers the necessary calibration adjustments to preserve near-correct coverage on the true labels. In what follows, we outline the procedure in more detail, leading to our first main theorem and the paper's contribution culminating in Theorem~\ref{theorem:main_ldpcpl}.

\begin{algorithm}
\caption{LDP-CP-L - Aggregator-side Procedure}
\label{alg:ldp_cp_labels_agg}
\begin{algorithmic}[1]
\STATE \textbf{Input:} Noisy calibration set $\{(x_i, \tilde{y}_i)\}_{i=1}^{n}$, privacy parameter $\epsilon$, number of classes $k$, target coverage $1-\alpha$, $\Delta$, $T$
\STATE \textbf{Output:} Threshold $q$ ensuring private coverage $1-\alpha-\Delta$ on true labels

\STATE Collect the noisy calibration set $\{(x_i, \tilde{y}_i)\}_{i=1}^{n}$.
\STATE Define functions $\hat{F}^n(q)$ and $\hat{F}^r(q)$ for candidate thresholds $q$:
\[ \hat{F}^n(q) = \frac{1}{n} \sum_{i=1}^{n} \mathbf{1}(S(x_i, \tilde{y}_i) \leq q), \quad
\hat{F}^r(q) = \frac{1}{n} \sum_{i=1}^{n} \frac{|C_q(x_i)|}{k}.
\]

\STATE Define:
\[ \hat{F}^c(q) = \frac{\hat{F}^n(q) - \beta \cdot \hat{F}^r(q)}{1 - \beta},
\hspace{0.4cm} s.t. \hspace{0.4cm}   \beta = \frac{k}{k-1+e^{\epsilon}}. \]
\STATE Initialize $s_{low}=0, s_{high}=1$
\FOR {$j=1,...,T$}
    \STATE Set $q^{(j)} = \frac{s_{low}+s_{high}}{2}$
    \STATE Obtain $Z^{(j)} = \hat{F}^c(q^{(j)})$
    \STATE \textbf{if} $Z^{(j)} \gt 1-\alpha + {\Delta}$ \textbf{then} $s_{high}=q^{(j)}$
    \STATE \textbf{else if} $Z^{(j)} \lt 1-\alpha $ \textbf{then} {$s_{low}=q^{(j)}$}
    \STATE \textbf{else} break
\ENDFOR
\STATE Return the threshold $q^{(j)}$
\end{algorithmic}
\label{alg:2}
\end{algorithm}

Next, we describe how the aggregator computes the noise-free CP threshold based on the noisy labels that were received. 
For each threshold $q$, let $F^c(q) = p\bigl(S(x,y)\le q\bigr)$, $F^n(q) = p\bigl(S(x,\tilde{y})\le q\bigr)$,  $F^r(q) = p\bigl(S(x,u)\le q\bigr), \text{ $u \sim U(1,\dots,k)$}$ coverage on true labels, noisy labels, and uniform respectively.
    It can be easily verified that
    \begin{equation}
    F^n(q) \;=\; (1-\beta)\,F^c(q)\;+\;\beta\,F^r(q).
    \label{fnq}
    \end{equation}
Given a noisy calibration set $\{(x_i,\tilde{y}_i)\}$, the aggregator can estimate $F^n(q)$ and $F^r(q)$ by
    \[
    \hat{F}^n(q)
    = \frac{1}{n}\sum_{i=1}^n {1}\bigl(S(x_i,\tilde{y}_i)\le q\bigr), \hspace{0.5cm} \hat{F}^r(q)
    = \frac{1}{n}\sum_{i=1}^n \frac{|C_q(x_i)|}{k}.
    \]
Thus, rearranging Eq. (\ref{fnq}) we obtain an estimation of $F^c(q)$:
\[
\hat{F}^c(q)
= \frac{\hat{F}^n(q) - \beta\hat{F}^r(q)}{\,1-\beta}.
\]

Hence, to find a threshold $q$ that yields coverage $1-\alpha$ on the \emph{true} labels in a private manner, we solve $\hat{F}^c(q)=1-\alpha$. In practice, we do a binary search over candidate thresholds which continues until either the estimate $\hat{Z}^{(j)}$ satisfies $1-\alpha \le \hat{Z}^{(j)} \le  1-\alpha + \Delta$ or the interval length $s_{high} - s_{low}$ becomes smaller than a predefined threshold $\tau$. This is exactly the ``noise-aware'' threshold that corrects for the $k$-RR noise mechanism. User's side and aggregator's side procedures are depicted in Algorithms \ref{alg:1} and \ref{alg:2} respectively.

\citet{Penso_2024} named this procedure  \textbf{NACP} (Noise-Aware Conformal Prediction).  While they were motivated by problems related to a general noisy channel (e.g. experts' mistakes/disagreements as to the true label), here we use the noisy channel as a privacy protection for the labels in the calibration set. As it turns out, using $k$-RR falls neatly into their paradigm and in turn yields Theorem~\ref{theorem:main_ldpcpl}.
Note that the $k$-RR mechanism is implemented here, instead of more advanced LDP methods, such as RAPPOR \citep{erlingsson2014rappor}, since it aligns well with the NACP framework.

\subsection{Local-DP Conformal Prediction on Scores (LDP-CP-S)}
\label{subsec:ldp-cp-scores}

In some scenarios, users may be able to compute the \emph{full conformity score} locally. Concretely, the aggregator (or model provider) sends the neural network’s parameters or logits to each user, who then computes the conformity score \(S(x_i, y_i)\) for the ground-truth label \(y_i\) on their own.
This approach leverages the ability to estimate quantiles of a distribution using noisy scores that are locally privatized by the users.
We employ the LDP-binary search algorithm described in  \citet{gaboardi2019locallyprivatemeanestimation} to estimate the $(1-\alpha)$-quantile of the scores.
This method is incorporated into the conformal prediction framework by estimating the $(1-\alpha)$-quantile of the conformity scores derived locally by the users. Once the quantile is estimated, it is used as the threshold for constructing prediction sets. The use of LDP ensures that the process is privacy-preserving, whereas the quantile estimation guarantees accurate calibration.
By combining the strengths of binary search and randomized response, this method offers a robust approach to privacy-preserving conformal prediction that is both practical and theoretically sound.
The settings are depicted in Figure \ref{fig:main_pipeline_s}. We next unpack LDP-CP-S by first describing the user’s side procedure that applies to
its data point, and then the procedure done by the untrusted aggregator to compute the
conformal prediction results.

\textbf{User-side Procedure.} For each user \(i\) with pair \((x_i, y_i)\):
\begin{enumerate}
    \item Compute the conformity score \( S_i = S(x_i, y_i) \) locally.
    \item Compare \( S_i \) with a threshold \( q(j) \) provided by the aggregator. (Note that each user has a single interaction with the aggregator -- see details in Aggregator's procedure).
    \item Return a binary response using randomized response (RR), ensuring \( \varepsilon \)-LDP.
\end{enumerate}

\begin{figure*}[t!]
    \centering
    \includegraphics[trim=0cm 1cm 0cm 0cm, width=0.80\linewidth]{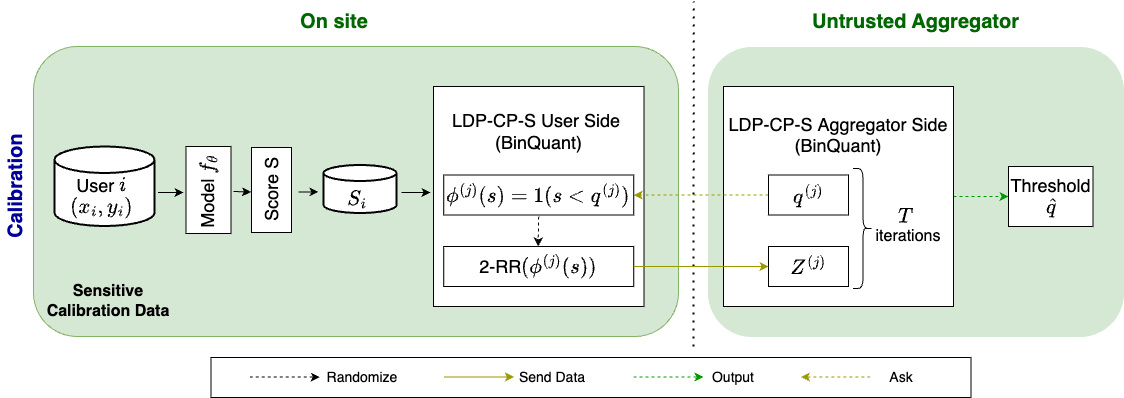}
    \caption{Local Differential Private Conformal Prediction on Scores (LDP-CP-S).}
    \label{fig:main_pipeline_s}
\end{figure*}
\begin{algorithm}
\caption{LDP-CP-S - both sides}
\label{alg:ldp_cp_s}
\begin{algorithmic}[1]
\STATE \textbf{Input:} Calibration set $\{(x_i, y_i)\}_{i=1}^{n}$, privacy parameter $\epsilon$, number of classes $k$, desired coverage $1-\alpha$, $\Delta$, $T$
\STATE \textbf{Output:} Threshold $q$ ensuring private coverage $1-\alpha-\Delta$ on true labels
\STATE Initialize $n'=\frac{n}{T}, s_{low}=0, s_{high}=1$
\FOR {$j=1,...,T$}
    \STATE Select users $\mathcal{U}^{(j)}=\{j \cdot n' + 1, j \cdot n' + 2,..., (j + 1) \cdot n'\}$
    \STATE Set $q^{(j)} = \frac{s_{low}+s_{high}}{2}$
    \STATE Individual users $i\in \mathcal{U}^{(j)}$: (1) compute their score $s=s(x_i, y_i)$, (2) sets $b_i = 1(s \lt q^{(j)})$, (3) sends the aggregator $RR(b_i)$
    \STATE Obtain $Z^{(j)} = \frac{e^\epsilon + 1}{e^\epsilon - 1}\cdot \frac{1}{n}\sum_{i\in \mathcal{U}^{(j)}} RR(b_i) - \frac{1}{e^\epsilon - 1}$
    \STATE \textbf{if} $Z^{(j)} \gt 1-\alpha + \Delta$ \textbf{then} $s_{high}=q^{(j)}$
    \STATE \textbf{else if} $Z^{(j)} \lt 1-\alpha  $ \textbf{then} {$s_{low}=q^{(j)}$}
    \STATE \textbf{else} break
\ENDFOR
\STATE Return $q^{(j)}$
\end{algorithmic}
\label{alg:3}
\end{algorithm}

\textbf{Aggregator Procedure.} The aggregator collects the binary responses from users over a binary search procedure. The process starts by defining an initial range that is guaranteed to contain the desired $(1 - \alpha)$-quantile. The range is repeatedly divided in half through a binary search procedure. At each step $j$, a midpoint $q(j)$ is calculated, and a subset of the users is used to privately estimate how many data points fall below this midpoint. This estimation is done using randomized response, which ensures that the algorithm complies with privacy guarantees. Depending on the results of this estimation, the algorithm updates the range: if too many data points are estimated to fall below the midpoint, the upper boundary is adjusted; if too few, the lower boundary is adjusted. Note that the subsets of users used at each step are disjoint, which assures that each user has at most one interaction with the aggregator. This process continues until either the estimate $\hat{Z}^{(j)}$ satisfies $1-\alpha \le \hat{Z}^{(j)} \le 1-\alpha + \Delta$ or the interval length $s_{high} - s_{low}$ becomes smaller than a predefined threshold $\tau$.
In Theorem~\ref{theorem:main_ldpcps} we give the concrete sample complexity, under which we can estimate $\hat{Z}^{(j)}$ both privately and accurately in all iterations of the binary search.
Once the aggregator finds the desired estimation, denoted \( \hat{q} \), the aggregator can now construct the prediction set as:
\(
   C_{\hat{q}}(x) \;=\; \{\, y \mid S(x,y) \le \hat{q}\}.
\)
User's side and aggregator's side procedures depicted in Algorithm box \ref{alg:3}.


\subsection{Theoretical Guarantees}
We now present our main theoretical results. Theorem \ref{theorem:main_ldpcpl} and \ref{theorem:main_ldpcps} describe the theoretical guarantees of LDP-CP-L and LDP-CP-S, respectively.
\begin{theorem}[\textbf{LDP-CP-L}]
    Fix $\alpha, \delta, \Delta>0$. There exists an $\epsilon$-local differentially private algorithm that draws $n = O\left(\frac{\log(\nicefrac{1}{\delta})}{\Delta^2h^2}\right)$ exchangeable samples from any admissible distribution $\mathcal{D}$, where $h = \frac{1-\beta}{1+\beta}$, and $\beta = \frac{k}{k-1+e^{\epsilon}}$, and, within at most $T = \lceil\log(\nicefrac{1}{\tau})\rceil$ iterations, produces an estimate $\hat{q}$ that satisfies $p(y\in C_{\hat{q}}(x) ) \;\ge\; 1-\alpha - \Delta$ with probability at least $1 - \delta$ (over the calibration set), where $1-\alpha$ is the desired coverage and $\tau$ is an a-priori bound on the length of an interval that can hold $\Delta$-probability mass.
\label{theorem:main_ldpcpl}
\end{theorem}
\begin{proof}
The LDP-CP mechanism applies k-RR to the input data, ensuring $\epsilon$-local differential privacy \citep{kairouz2016discrete}.
Additionally, using the post-processing property of differential privacy \citep{dwork2006differential}, it is safe to perform arbitrary computations on the output of a differentially private mechanism - which maintains the privacy guarantees of the mechanism. Therefore, since $\text{k-RR}(\{x_i,y_i\}_{i=1}^{n})$ satisfies $\epsilon$-local-differential privacy, and because NACP \citep{Penso_2024} is a deterministic or randomized post-processing function, it follows that NACP(k-RR$(\{x_i,y_i\}_{i=1}^{n})$) satisfies $\epsilon$-local-differential privacy.

The remainder of the proof focuses on the conformal prediction coverage guarantee bound. Given our k-RR($\epsilon$) $\epsilon$-LDP mechanism, we derive the noise rate $\beta = \frac{k}{k-1+e^{\epsilon}}$. Substituting $\beta$, and $n$ into Theorem \ref{theorem:nacp} we obtain $\Delta(n, \beta, \delta)$ such that $ p(y\in C_{\tilde q}(x)) \;\ge\; 1-\alpha - \Delta.$
\end{proof}

\begin{theorem}[\textbf{LDP-CP-S}]
Fix $\alpha, \delta, \Delta>0$. There exists an $\epsilon$-local differentially private algorithm that draws $n = O\left(\frac{T}{\Delta^2}(\frac{e^\epsilon + 1}{e^\epsilon - 1})^2\log(\nicefrac{T}{\delta}))\right)$ exchangeable samples from any admissible distribution $\mathcal{D}$ and, within at most $T = \lceil\log(\nicefrac{1}{\tau})\rceil$ iterations, produces an estimate $\hat{q}$ that satisfies $p(y\in C_{\hat{q}}(x)) \;\ge\; 1-\alpha - \Delta$ with probability at least $1 - \delta$, where $1-\alpha$ is the desired coverage and $\tau$ is an a-priori bound on the length of an interval that can hold $\Delta$-probability mass.
\label{theorem:main_ldpcps}
\end{theorem}

\begin{proof}
    The privacy proof of a $\epsilon$-LDP quantile binary search algorithm can be found in \citet{gaboardi2019locallyprivatemeanestimation}. Users start by computing scores locally and then a local differentially private quantile binary search algorithm is taken place \citep{gaboardi2019locallyprivatemeanestimation}.
    $\epsilon$-LDP follows immediately from the fact that the only time we access the data is via randomized response. The output of the algorithm is the $(1-\alpha)$'th quantile that one would obtain by applying conformal prediction to the clean data, yet it is recovered solely from the privatized (noisy) scores.
\end{proof}

\section{Practical Considerations}

\subsection{LDP-CP-L vs. LDP-CP-S}
\label{sec:s_vs_l}

The proposed methods, LDP-CP-L and LDP-CP-S, offer distinct approaches for achieving local differential privacy in conformal prediction, each is tailored to specific privacy setups and has different coverage guarantees. An in-depth understanding of their trade-offs is essential to determine their suitability for various scenarios.

LDP-CP-L focuses on achieving local differential privacy by perturbing the labels ($y$) while exposing the features ($x$) to the untrusted aggregator. This approach ensures privacy for the labels, which are typically considered more sensitive in many applications. This mechanism is particularly advantageous in scenarios where users are resource-constrained, since they only need to perturb their labels locally before sending $(x, \tilde{y})$ to the aggregator. This design also keeps the model parameters and scoring functions secret from the users, because all computations related to the score are performed centrally. Consequently, LDP-CP-L is well-suited for real-world calibration datasets of moderate size, such as those containing several thousand records, where the additional noise introduced by the mechanism remains manageable. However, a notable drawback of LDP-CP-L is its sensitivity to the number of classes $k$ in the dataset. As $k$ increases, the calibration error ($\Delta$) grows, potentially compromising coverage guarantees for datasets with a large number of classes. Additionally, while the exposure of $x$ provides practical utility by allowing centralized score computation, it introduces privacy concerns. This drawback can be partially mitigated by employing the shuffle model of differential privacy, which adds an extra layer of anonymity to the users' data (See further discussion in Section~\ref{sec:on_exp_x_main}). Furthermore, privatizing $x$ would result in a substantial insertion of noise, thereby leading to a significant degradation in the accuracy of the algorithms.

By contrast, LDP-CP-S achieves privacy at the score level by having users compute scores locally and perturb their responses before submitting them to the aggregator. This design ensures that both $x$ and $y$ remain private and are never exposed to the aggregator, making LDP-CP-S particularly suitable for scenarios where feature privacy is paramount. Unlike LDP-CP-L, the performance of LDP-CP-S is independent of the number of classes $k$. However, this approach imposes additional computational requirements on the users, who must perform local computations involving the model and scoring function. This requirement necessitates sharing the model with users, which might raise concerns about intellectual property or model misuse. Furthermore, LDP-CP-S requires a larger calibration dataset ($n$) to achieve a sufficiently small calibration error, potentially limiting its applicability in scenarios with limited data availability.

In summary, the choice between LDP-CP-L and LDP-CP-S depends on the specific privacy and computational constraints of the application. LDP-CP-L is better suited for scenarios where label privacy is the primary concern, datasets are of moderate size (and larger). Its design minimizes user-side computations and protects the model from exposure. Conversely, LDP-CP-S is preferable when both feature and label privacy are critical, and when users have the computational resources to perform local scoring. Its robustness to the number of classes makes it a strong candidate for applications with a large class set, provided a sufficiently large calibration dataset is available. By carefully considering these trade-offs, practitioners can select the most appropriate method for their privacy-preserving conformal prediction tasks. In the experiment section, we report a numerical comparison of methods accuracy ($\Delta_L, \Delta_S$) as a function of the calibration set size $n$ and the number of classes $k$ (Figure \ref{fig:s_vs_l}).

\subsection{The Shuffle Model of Differential Privacy}
\label{sec:on_exp_x_main}
The shuffle model of differential privacy enhances privacy guarantees by introducing an additional layer of anonymization between users and the data aggregator. In this model, each user applies a local randomizer to their data and then sends the output to a secure shuffler, which permutes the messages uniformly at random before forwarding them to the aggregator. This anonymization mechanism breaks the association between individual users and their messages, thereby amplifying privacy guarantees beyond those attainable in the purely local model.

A notable benefit of the shuffle model is its capacity for \emph{privacy amplification}~\citep{cheu2022differentialprivacyshufflemodel}. If each user applies an $\varepsilon$-LDP mechanism before sending their message, the effective privacy loss can be reduced to approximately $\epsilon^{\text{eff}}=\nicefrac{\varepsilon}{\sqrt{n}}$ in the shuffled output. This amplification enables stronger privacy guarantees with the same local noise or, conversely, allows for reduced noise to achieve a given privacy target—thereby improving utility in downstream tasks. This is particularly beneficial in regimes where moderately high local privacy levels (e.g., $\varepsilon > 1$) would otherwise impose significant performance degradation.

Given these advantages, our approach incorporates the shuffle model to improve the privacy-utility trade-off of our mechanisms. We found that adopting the shuffle model allows us to retain decentralization and local control while achieving significantly improved accuracy through amplification. Importantly, this integration does not alter the algorithmic structure of our mechanisms but rather augments their privacy analysis and performance guarantees under realistic assumptions of an honest-but-curious shuffler. As such, the shuffle model serves not only as a technical enhancement but also as a practical enabler of more effective private learning in our setting.

\section{Experiments}
\label{sec:exps}

In this section, we evaluate the capabilities of our LDP-CP algorithms on various medical imaging datasets and address the utility-coverage tradeoff.

 \textbf{Compared Methods.} Our method takes an existing conformity score $S$ and computes a threshold $q$ that considers the injected noise level.  We implemented two popular conformal prediction scores, namely APS
\citep{romano2020classification} and HPS \citep{vovk2005conformal}. For each score $S$, we compared the following CP  methods: (1) Not-Private-CP  with coverage guarantee $1-\alpha$ -  using a calibration set with clean labels (2) LDP-CP-\{S,L\} with coverage guarantee $1-\alpha - \Delta$. and (3) LDP-CP-\{S,L\}* with coverage guarantee $1-\alpha$. We share our code for  reproducibility\footnote{\url{https://github.com/cobypenso/local-differential-private-conformal-prediction}}.

\begin{figure}[H]
    \centering
    \includegraphics[trim=0cm 0.5cm 0cm 0cm, width=0.95\linewidth]{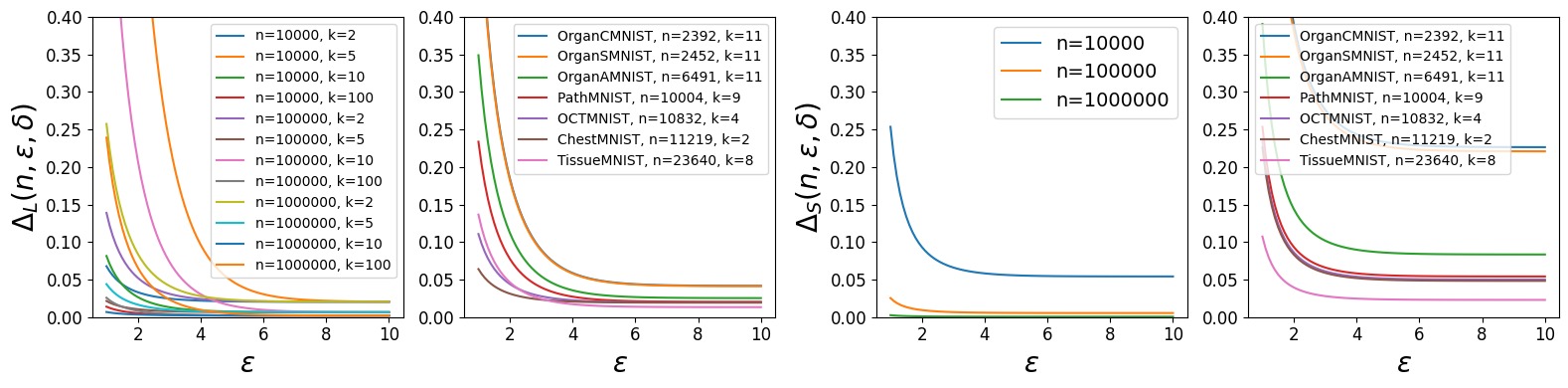}
    \quad \quad  \quad \quad \quad  \quad \quad  \quad \quad \quad \quad \quad \quad \quad \quad \quad \quad \quad \quad \quad  \quad \quad  \quad   \quad (a) LDP-CP-L - $\Delta_L$ \quad \quad \quad  \quad \quad \quad \quad \quad \quad \quad \quad \quad (b) LDP-CP-S - $\Delta_S$ \\
    \caption{CP correction terms $\Delta_L,\Delta_S$ as a function of $\epsilon$ privacy parameter across different dataset configurations of $n$ and $k$ \textbf{without the shuffle model}.}
    \label{fig:tradeoff-general-l-and-s}
    \includegraphics[trim=0cm 0.5cm 0cm 0cm, width=0.95\linewidth]{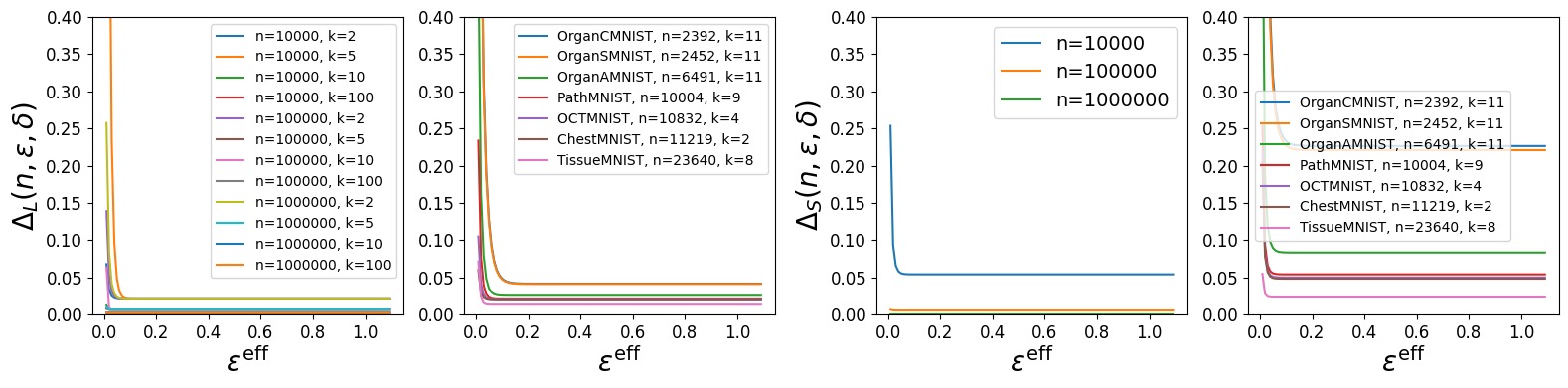}
    \quad \quad  \quad \quad \quad  \quad \quad  \quad \quad \quad \quad \quad \quad \quad \quad \quad \quad \quad \quad \quad  \quad \quad  \quad   \quad (a) LDP-CP-L - $\Delta_L$ \quad \quad \quad  \quad \quad \quad \quad \quad \quad \quad \quad \quad (b) LDP-CP-S - $\Delta_S$ \\
    \caption{CP correction terms $\Delta_L,\Delta_S$ as a function of $\epsilon^{\text{eff}}$ privacy parameter across different dataset configurations of $n$ and $k$ \textbf{with the shuffle model}.}
    \label{fig:tradeoff-general-l-and-s-shuffle}
\end{figure}

\begin{figure*}[ht!]
    \centering
    \includegraphics[trim=0cm 1.5cm 0cm 0cm, width=1.02\linewidth]{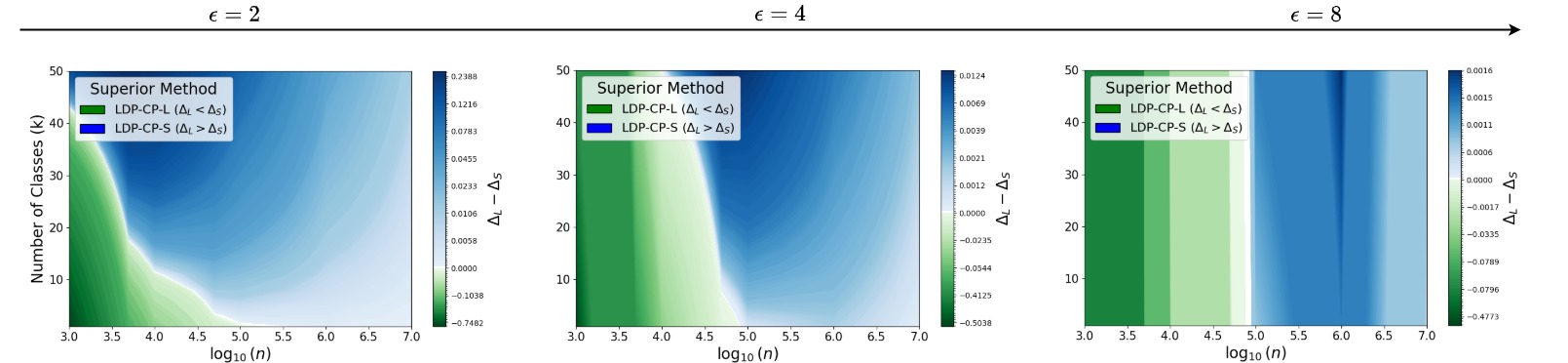}
    \caption{Comparison of $\Delta_L$ and $\Delta_S$ as a function of the number of classes $k$ and dataset size $n$, for $\epsilon=2,4,8$.}
    \label{fig:s_vs_l}
\end{figure*}

 \textbf{Datasets.}
We present results on several publicly available medical imaging classification datasets \citep{medmnistv2}. 
 The \textbf{TissuMNIST} dataset \citep{medmnistv1,medmnistv2}
 contains 236,386 human kidney cortex cells,  organized into 8 categories. Each gray-scale
image is $32 \times 32 \times 7$ pixels. The  2D projections were obtained by taking the maximum pixel value along the axial-axis of each pixel, and were resized into $28 \times 28$ gray-scale images \citep{woloshuk2021situ}.
The  \textbf{OrganSMNIST} dataset \citep{medmnistv2}  contains 25,221 images of abdominal CT in eleven classes. The images are  $28 \times 28$ in size. Here, we used a train/calibration/test split of 13,940/2,452/8,829 images. \textbf{OrganAMNIST} and  \textbf{OrganCMNIST} are simliar datasets, the differences of Organ\{A,C,S\}MNIST are the views and dataset size. Lastly, \textbf{OCTMNIST} Retina is an OCT image dataset.

\begin{figure}[H]
    \centering
    \includegraphics[trim=0.5cm 1cm 0cm 0cm, width=0.30\textwidth]{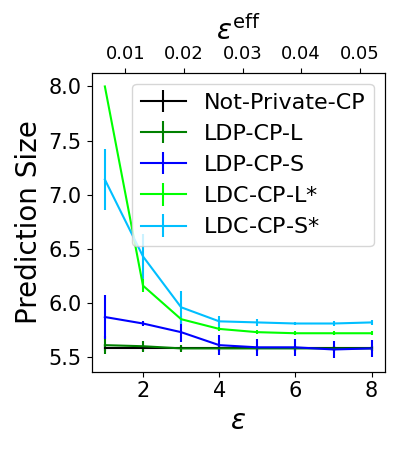} 
    \hspace{1cm}
    \includegraphics[trim=0.5cm 1cm 0cm 0cm, width=0.30\textwidth]{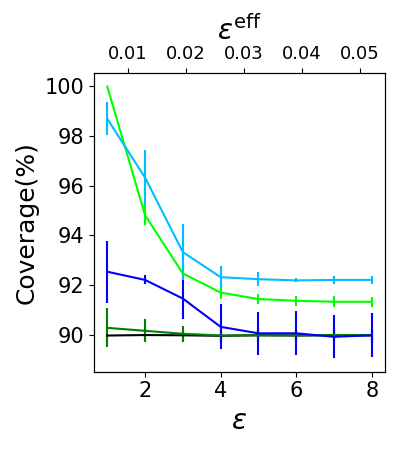}
    \caption{Size of prediction set (left) and coverage (right) as a function of the privacy $\epsilon$ (bottom x-axis) and effective privacy $\epsilon^{\text{eff}}$ (top x-axis). We show the (mean $\pm$ std) on TissueMNIST and APS score.}
    \label{fig:tissuemnist_calib_with_g}
\end{figure}

\textbf{Utility-Coverage Tradeoff}. 
\label{exp:utility_comp}
Theorems \ref{theorem:main_ldpcpl} and \ref{theorem:main_ldpcps} state that LDP-CP-L and LDP-CP-S are $\epsilon$-LDP with a conformal prediction coverage guarantee correction term of $\Delta_L(n, \beta(\epsilon,k), \delta)$ and $ \Delta_S(n,\epsilon,\delta)$, which depend on the number of samples n, the privacy $\epsilon$, $\delta$, and $\Delta_L$  also depends on the number of classes $k$. Figure \ref{fig:tradeoff-general-l-and-s} explores the trade-offs of $\epsilon,n,k$ on the conformal prediction correction terms $\Delta_L,\Delta_S$, on different configurations and various medical image datasets \citep{medmnistv2} respectively. The results show both the finite-sample practicality and theoretical asymptotic performance of LDP-CP on existing medical datasets that are medium in size and simulated scenarios with growing $n$. For the majority of the evaluated datasets with $\epsilon=3$, LDP-CP-L maintains correction terms that become negligible for $\alpha=0.1$ and higher ($1-\alpha$ is the coverage). $\Delta_L,\Delta_S \xrightarrow{} 0$ as $n \xrightarrow{} \infty$. While as the number of samples grows $\Delta_S$ goes to 0, for the evaluated medical datasets (with $n$ fixed) $\Delta_S>\Delta_L$, i.e. being inferior to $\Delta_L$, and is therefore more applicable to datasets with larger calibration sets.
Figure ~\ref{fig:tradeoff-general-l-and-s-shuffle} covers the same experiment setup, this time when the shuffle model is deployed and presents the correction terms as a function of the effective privacy. The effective privacy $\epsilon^{\text{eff}}$ ranges between 0 and 0.2, which is much more practical than the $\epsilon \geq 3$ acquired without the shuffle model.  

Figure \ref{fig:s_vs_l} compares $\Delta_{S}$ and $\Delta_{L}$ as a function of n, k, and $\epsilon$. Results show that roughly speaking for $n \geq 10^5 \xrightarrow{} \Delta_S \leq \Delta_L$. In addition, as the number of classes $k$ increases (particularly for large values of $k$, e.g., 100 or 1000), LDP-CP-S dominates, whereas LDP-CP-L exhibits deteriorating performance.

\textbf{Conformal Prediction Results}. 
Table \ref{tab:results_compact} reports the mean size and coverage when applying conformal prediction in a non $\epsilon$-LDP setting (Not-Private-CP), and when satisfying the local differentially private property. We report the results of two variations of our methods where the first consists of calibration using $1-\alpha$ and the second variation (denoted by $^*$) using $1-\alpha+\Delta$ in the calibration phase. For the experiment in Table \ref{tab:results_compact}, we used $\epsilon=4$ and $\alpha=0.1$ over 100 different data splits (seeds). Table \ref{tab:results_compact} also provides the effective privacy $\epsilon^{\text{eff}}$ per dataset to show increased practicality when the shuffle model is incorporated.  We can see that when we use $1-\alpha+\Delta$ in our calibration algorithm, we indeed obtain the required $1-\alpha$ coverage.

\begin{table}[t!]
\centering
\caption{Calibration results for HPS and APS conformal scores across various datasets, using $\epsilon=4$, $\epsilon^{\text{eff}}=\frac{\epsilon}{\sqrt{n}}$, and $\alpha=0.1$ on 100 different seeds.}
\label{tab:results_compact}
\scalebox{.8}{
\begin{tabular}{l|l|cccc}
\toprule
\textbf{Dataset} & \textbf{Method} & \multicolumn{2}{c}{\textbf{HPS}} & \multicolumn{2}{c}{\textbf{APS}} \\
\cmidrule(lr){3-4} \cmidrule(lr){5-6}
 & & size $\downarrow$ & coverage (\%) & size $\downarrow$ & coverage (\%) \\
\midrule
           & Not-Private-CP & 2.57 $\pm$ 0.03 & 90.06 $\pm$ 0.99 & 2.61 $\pm$ 0.03 & 90.06 $\pm$ 0.97 \\  
           & LDP-CP-L & 2.56 $\pm$ 0.04 & 89.99 $\pm$ 1.01 & 2.61 $\pm$ 0.03 & 90.02 $\pm$ 1.01 \\  
OCTMNIST  & LDP-CP-S & 2.58 $\pm$ 0.04 & 90.21 $\pm$ 0.63 & 2.67 $\pm$ 0.08 & 90.84 $\pm$ 1.45 \\  
 ($\epsilon^{\text{eff}}=0.038$)          & LDP-CP-L*& 2.76 $\pm$ 0.04 & 92.22 $\pm$ 0.92 &  2.79 $\pm$ 0.03 & 92.28 $\pm$ 0.87 \\ 
           & LDP-CP-S*& 2.97 $\pm$ 0.07 & 94.38 $\pm$ 0.81 &  2.99 $\pm$ 0.06 & 94.35 $\pm$ 0.70 \\ \hline 
           & Not-Private-CP & 5.55 $\pm$ 0.02 & 90.00 $\pm$ 0.24 & 5.58 $\pm$ 0.02 & 89.96 $\pm$ 0.24 \\  
           & LDP-CP-L & 5.54 $\pm$ 0.02 & 89.97 $\pm$ 0.29 & 5.58 $\pm$ 0.02 & 89.97 $\pm$ 0.27 \\  
TissueMNIST  & LDP-CP-S & 6.12 $\pm$ 0.01 & 95.35 $\pm$ 0.07 & 5.61 $\pm$ 0.09 & 90.32 $\pm$ 0.91 \\  
($\epsilon^{\text{eff}}=0.026$)           & LDP-CP-L*& 5.71 $\pm$ 0.02 & 91.68 $\pm$ 0.27 &  5.76 $\pm$ 0.02 & 91.70 $\pm$ 0.25 \\ 
           & LDP-CP-S*& 6.12 $\pm$ 0.01 & 95.35 $\pm$ 0.07 &  5.83 $\pm$ 0.05 & 92.32 $\pm$ 0.45 \\ \hline 
           & Not-Private-CP & 1.93 $\pm$ 0.05 & 90.09 $\pm$ 0.66 & 2.35 $\pm$ 0.05 & 90.10 $\pm$ 0.55 \\  
           & LDP-CP-L & 1.88 $\pm$ 0.07 & 89.49 $\pm$ 0.93 & 2.30 $\pm$ 0.09 & 89.63 $\pm$ 0.91 \\  
OrganSMNIST  & LDP-CP-S & 1.61 $\pm$ 0.21 & 84.81 $\pm$ 4.48 & 2.09 $\pm$ 0.07 & 87.38 $\pm$ 1.37 \\  
($\epsilon^{\text{eff}}=0.080$)           & LDP-CP-L*& 2.77 $\pm$ 0.22 & 95.35 $\pm$ 0.74 &  3.35 $\pm$ 0.22 & 95.45 $\pm$ 0.74 \\ 
           & LDP-CP-S*& 3.90 $\pm$ 0.03 & 97.75 $\pm$ 0.06 &  4.75 $\pm$ 0.03 & 98.40 $\pm$ 0.07 \\ \hline 
           & Not-Private-CP & 1.19 $\pm$ 0.02 & 89.99 $\pm$ 0.46 & 1.61 $\pm$ 0.02 & 90.02 $\pm$ 0.38 \\  
           & LDP-CP-L & 1.31 $\pm$ 0.00 & 92.17 $\pm$ 0.09 & 1.60 $\pm$ 0.03 & 89.94 $\pm$ 0.54 \\  
OrganAMNIST  & LDP-CP-S & 1.15 $\pm$ 0.05 & 88.89 $\pm$ 0.96 & 1.67 $\pm$ 0.21 & 90.35 $\pm$ 3.10 \\  
($\epsilon^{\text{eff}}=0.049$)           & LDP-CP-L*& 1.43 $\pm$ 0.03 & 93.62 $\pm$ 0.34 &  1.89 $\pm$ 0.05 & 93.51 $\pm$ 0.51 \\ 
           & LDP-CP-S*& 1.88 $\pm$ 0.19 & 96.52 $\pm$ 0.82 &  2.44 $\pm$ 0.19 & 96.80 $\pm$ 0.60 \\ \hline 
           & Not-Private-CP & 1.18 $\pm$ 0.03 & 89.99 $\pm$ 0.71 & 1.56 $\pm$ 0.03 & 90.02 $\pm$ 0.65 \\  
           & LDP-CP-L & 1.30 $\pm$ 0.00 & 91.96 $\pm$ 0.13 & 1.52 $\pm$ 0.04 & 89.36 $\pm$ 0.91 \\  
OrganCMNIST  & LDP-CP-S & 0.95 $\pm$ 0.08 & 82.44 $\pm$ 2.74 & 1.39 $\pm$ 0.07 & 86.59 $\pm$ 2.44 \\  
($\epsilon^{\text{eff}}=0.081$)           & LDP-CP-L*& 1.63 $\pm$ 0.10 & 95.21 $\pm$ 0.74 &  2.05 $\pm$ 0.13 & 95.21 $\pm$ 0.80 \\ 
           & LDP-CP-S*& 2.47 $\pm$ 0.02 & 98.18 $\pm$ 0.06 &  2.90 $\pm$ 0.04 & 98.15 $\pm$ 0.10 \\ \hline 

\bottomrule
\end{tabular}
}

\end{table}

Next, we experimented with different values of $\epsilon$, namely $\epsilon \in [1,8]$. Recall that with the shuffle model the effective privacy is $\epsilon^{\text{eff}}=\frac{\epsilon}{\sqrt{n}}$. Figure \ref{fig:tissuemnist_calib_with_g} shows the coverage, size of the prediction set, and effective privacy on the TissueMNIST dataset. As expected as $\epsilon$ grows, $\Delta$ decreases and CP results get closer to the Non-Private-CP. In addition, LDP-CP-L performs on a par with the Non-Private-CP for all $\epsilon$'s, and LDP-CP-S for $\epsilon^{\text{eff}}\geq 0.03 $ ($\epsilon\geq4$), showcasing their true applicability.
\section{Conclusion and Open Problems}

We presented two complementary CP approaches that explicitly account for local differential privacy on labels and scores. By leveraging different randomized response mechanisms, we incorporate them into the CP procedure to provide \emph{valid coverage} on the \emph{true} labels, while guaranteeing that $\epsilon$-local DP protects each user's label. For LDP-CP-L, this involves label perturbation and noise-aware calibration, making it suitable for privacy-sensitive settings where user computation is limited. For LDP-CP-S, the method allows users to compute and perturb scores locally, enabling calibration in scenarios where model access is feasible. Our framework bridges the gap between robust uncertainty quantification and strong local privacy for user data. It provides theoretical guarantees for both approaches while addressing practical considerations in privacy-sensitive classification tasks. By offering guidelines for selecting  LDP-CP-L or LDP-CP-S, we empower practitioners to balance privacy, computational constraints, and performance requirements, thus enabling secure and accurate predictions in real-world applications.

An important open problem is the integration of more advanced methods in Local Differential Privacy (LDP), such as RAPPOR or the mechanisms proposed by \citet{bassily2017practicallocallyprivateheavy}, into the LDP-CP-L framework. A key challenge in this integration lies in determining whether the aggregator can effectively process data from users who submit a potentially large subset of labels. Addressing this issue would enhance the robustness and scalability of the LDP-CP-L framework, enabling it to accommodate more complex privacy-preserving mechanisms while maintaining accuracy and efficiency.

Another open problem arises from the consideration of a more stringent setup for the LDP-CP problem, in which computing scores on the user's side is prohibited (as in LDP-CP-L) while simultaneously ensuring the privacy of both the signal/features $x$ and the label $y$. In this scenario, neither the LDP-CP-L nor the LDP-CP-S approaches are suitable, highlighting a critical gap in existing methodologies. Developing privacy-preserving solutions that can operate under such restrictive conditions without compromising utility remains an open research challenge and an important direction for future work.

\bibliography{refs}

@string {nips = {Advances in Neural Information Processing Systems (NeurIPs)}}

@string {miccai={International Conference on Medical Image Computing and Computer-Assisted Intervention (MICCAI)}}

@string {isbi = {The IEEE International Symposium on Biomedical Imaging (ISBI)}}

@article{olsson2022estimating,
  title={Estimating diagnostic uncertainty in artificial intelligence assisted pathology using conformal prediction},
  author={Olsson, Henrik and Kartasalo, Kimmo and Mulliqi, Nita and others},
  journal={Nature Communications},
  volume={13},
  number={1},
  pages={7761},
  year={2022}
}

@article{sesia2023adaptive,
  title={Adaptive conformal classification with noisy labels},
  author={Sesia, Matteo and Wang, YX and Tong, Xin},
  journal={arXiv preprint arXiv:2309.05092},
  year={2023}
}

@article{medmnistv2,
    title={{MedMNIST} v2-A large-scale lightweight benchmark for {2D} and {3D} biomedical image classification},
    author={Yang, Jiancheng and Shi, Rui and Wei, Donglai and Liu, Zequan and Zhao, Lin and Ke, Bilian and Pfister, Hanspeter and Ni, Bingbing},
    journal={Scientific Data},
    volume={10},
    number={1},
    pages={41},
    year={2023}
}

@inproceedings{ghazi2021,
  title={Deep learning with label differential privacy},
  author={Ghazi, Badih and Golowich, Noah and Kumar, Ravi and Manurangsi, Pasin and Zhang, Chiyuan},
  booktitle=nips,
    year={2021}
}

@article{romano2020classification,
  title={Classification with valid and adaptive coverage},
  author={Romano, Yaniv and Sesia, Matteo and Candes, Emmanuel},
  journal={Advances in Neural Information Processing Systems},
  year={2020}
}

@article{angelopoulos2023conformal,
  title={Conformal prediction: A gentle introduction},
  author={Angelopoulos, Anastasios N and Bates, Stephen and others},
  journal={Foundations and Trends in Machine Learning},
  volume={16},
  number={4},
  pages={494--591},
  year={2023}
  }

@article{woloshuk2021situ,
  title={In situ classification of cell types in human kidney tissue using {3D} nuclear staining},
  author={Woloshuk, Andre and Khochare, Suraj and Almulhim, Aljohara F and McNutt, Andrew T and Dean, Dawson and Barwinska, Daria and Ferkowicz, Michael J and Eadon, Michael T and Kelly, Katherine J and Dunn, Kenneth W and others},
  journal={Cytometry Part A},
  volume={99},
  number={7},
  pages={707--721},
  year={2021},
  publisher={Wiley Online Library}
}

@inproceedings{lu2022improving,
  title={Improving trustworthiness of {AI} disease severity rating in medical imaging with ordinal conformal prediction sets},
  author={Lu, Charles and Angelopoulos, Anastasios N and Pomerantz, Stuart},
  booktitle=miccai,
    year={2022}
}

@inproceedings{lu2022fair,
  title={Fair conformal predictors for applications in medical imaging},
  author={Lu, Charles and Lemay, Andr{\'e}anne and Chang, Ken and H{\"o}bel, Katharina and Kalpathy-Cramer, Jayashree},
  booktitle={Proceedings of the AAAI Conference on Artificial Intelligence},
   year={2022}
}

@book{vovk2005conformal,
  title={Algorithmic learning in a random world},
  author={Vovk, Vladimir and Gammerman, Alexander and Shafer, Glenn},
  volume={29},
  year={2005},
  publisher={Springer}
}

@inproceedings{Penso_2024,
  title={Conformal prediction of classifiers with many classes based on noisy labels},
  author={Penso, Coby and Goldberger, Jacob and Fetaya, Ethan},
   booktitle={Proceedings of the Symposium on Conformal and Probabilistic Prediction with Applications},
  year={2025}
}

@inproceedings{Penso_20242,
title={A Conformal Prediction Score that  is Robust to Label Noise},
author = {Penso, Coby and Goldberger, Jacob},
  booktitle={MICCAI Int. Workshop on Machine Learning in Medical Imaging (MLMI)}, 
  year={2024}
}

@article{einbinder2022conformal,
  title={Conformal Prediction is Robust to Label Noise},
  author={Einbinder, Bat-Sheva and Bates, Stephen and Angelopoulos, Anastasios N and Gendler, Asaf and Romano, Yaniv},
  journal={arXiv preprint arXiv:2209.14295},
  year={2022}
}

@inproceedings{medmnistv1,
    title={{MedMNIST} Classification Decathlon: A Lightweight AutoML Benchmark for Medical Image Analysis},
    author={Yang, Jiancheng and Shi, Rui and Ni, Bingbing},
    booktitle=isbi,
    year={2021}
}

@article{angelopoulos2022private,
  title={Private prediction sets},
  author={Angelopoulos, Anastasios Nikolas and Bates, Stephen and Zrnic, Tijana and Jordan, Michael I},
  year={2022},
  publisher={PubPub},
  journal={Harvard Data Science Review}
}

@article{warner1965randomized,
  title={Randomized response: A survey technique for eliminating evasive answer bias},
  author={Warner, Stanley L},
  journal={Journal of the American statistical association},
  volume={60},
  number={309},
  pages={63--69},
  year={1965},
  publisher={Taylor \& Francis}
}

@inproceedings{kairouz2016discrete,
  title={Discrete distribution estimation under local privacy},
  author={Kairouz, Peter and Bonawitz, Keith and Ramage, Daniel},
  booktitle={International Conference on Machine Learning},
  pages={2436--2444},
  year={2016},
  organization={PMLR}
}

@inproceedings{wang2017locally,
  title={Locally differentially private protocols for frequency estimation},
  author={Wang, Tianhao and Blocki, Jeremiah and Li, Ninghui and Jha, Somesh},
  booktitle={26th USENIX Security Symposium (USENIX Security 17)},
  pages={729--745},
  year={2017}
}

@inproceedings{duchi2013local,
  title={Local privacy and statistical minimax rates},
  author={Duchi, John C and Jordan, Michael I and Wainwright, Martin J},
  booktitle={2013 IEEE 54th annual symposium on foundations of computer science},
  pages={429--438},
  year={2013},
  organization={IEEE}
}

@misc{gaboardi2019locallyprivatemeanestimation,
      title={Locally Private Mean Estimation: Z-test and Tight Confidence Intervals}, 
      author={Marco Gaboardi and Ryan Rogers and Or Sheffet},
      year={2019},
      eprint={1810.08054},
      archivePrefix={arXiv},
      primaryClass={cs.DS},
      url={https://arxiv.org/abs/1810.08054}, 
}

@inproceedings{erlingsson2014rappor,
  title={Rappor: Randomized aggregatable privacy-preserving ordinal response},
  author={Erlingsson, {\'U}lfar and Pihur, Vasyl and Korolova, Aleksandra},
  booktitle={Proceedings of the 2014 ACM SIGSAC conference on computer and communications security},
  pages={1054--1067},
  year={2014}
}

@misc{apple2017privacy,
  author       = {Apple},
  title        = {Learning with Privacy at Scale},
  year         = {2017},
  url          = {https://docs-assets.developer.apple.com/ml-research/papers/learning-with-privacy-at-scale.pdf},
  note         = {Accessed: [Insert Access Date]}
}

@inproceedings{beimel2013private,
  title={Private learning and sanitization: Pure vs. approximate differential privacy},
  author={Beimel, Amos and Nissim, Kobbi and Stemmer, Uri},
  booktitle={International Workshop on Approximation Algorithms for Combinatorial Optimization},
  pages={363--378},
  year={2013},
  organization={Springer}
}

@inproceedings{dwork2006differential,
  title={Differential privacy},
  author={Dwork, Cynthia},
  booktitle={International Colloquium on Automata, Languages, and Programming},
  pages={1--12},
  year={2006},
  organization={Springer}
}

@article{kasiviswanathan2011can,
  title={What can we learn privately?},
  author={Kasiviswanathan, Shiva Prasad and Lee, Homin K and Nissim, Kobbi and Raskhodnikova, Sofya and Smith, Adam},
  journal={SIAM Journal on Computing},
  volume={40},
  number={3},
  pages={793--826},
  year={2011},
  publisher={SIAM}
}

@misc{bassily2017practicallocallyprivateheavy,
      title={Practical Locally Private Heavy Hitters}, 
      author={Raef Bassily and Kobbi Nissim and Uri Stemmer and Abhradeep Thakurta},
      year={2017},
      eprint={1707.04982},
      archivePrefix={arXiv},
      primaryClass={cs.DS},
      url={https://arxiv.org/abs/1707.04982}, 
}

@misc{cheu2022differentialprivacyshufflemodel,
      title={Differential Privacy in the Shuffle Model: A Survey of Separations}, 
      author={Albert Cheu},
      year={2022},
      eprint={2107.11839},
      archivePrefix={arXiv},
      primaryClass={cs.CR},
      url={https://arxiv.org/abs/2107.11839}, 
}

@inproceedings{Clarkson2024,
author={ Clarkson, Jase and  Xu, Wenkai  and  i Cucuringu, Mihai and  Reinert, Gesine},
title={Split Conformal Prediction under Data Contamination},
booktitle={Proceedings of the Thirteenth Symposium on Conformal and Probabilistic Prediction with Applications},
year={2024}
}

\end{document}